\pdfoutput=1
\documentclass[letterpaper,journal]{IEEEtran}
\usepackage{cite}
\usepackage{amsmath,amssymb,amsfonts}
\usepackage{array}
\usepackage{booktabs}
\usepackage{diagbox}
\usepackage{multirow}
\usepackage{graphicx}
\usepackage{float}
\usepackage{threeparttable}
\usepackage{color,xcolor}
\usepackage{url}
\usepackage{microtype}
\usepackage{makecell}
\usepackage[ruled,vlined,linesnumbered]{algorithm2e}
\SetAlgoNlRelativeSize{0.5}
\usepackage[caption=false,font=footnotesize]{subfig}

\usepackage[colorlinks,
linkcolor=blue,
anchorcolor=blue,
citecolor=blue]{hyperref}
\usepackage{textcomp}
\usepackage{stfloats}
\usepackage{balance}
\newtheorem{definition}{Definition}
\newtheorem{proof}{Proof}

\newtheorem{remark}{Remark}
\newtheorem{property}{Property}
\newtheorem{proposition}{Proposition}

\def\BibTeX{{\rm B\kern-.05em{\sc i\kern-.025em b}\kern-.08em
    T\kern-.1667em\lower.7ex\hbox{E}\kern-.125emX}}

\begin{document}

\title{A Boundary-Aware Non-parametric Granular-Ball Classifier Based on Minimum Description Length}

\author{Zeqiang~Xian, Caihui~Liu\IEEEauthorrefmark{1}, Yong~Zhang, Wenjing~Qiu, Duoqian~Miao, and Witold~Pedrycz~\IEEEmembership{Life Fellow,~IEEE}%
	\thanks{Zeqiang Xian, Caihui Liu, Yong Zhang, and Wenjing Qiu are with the Department of Mathematics and Computer Science, Gannan Normal University, Ganzhou 341000, Jiangxi, China, and also with the Key Laboratory of Data Science and Artificial Intelligence of Jiangxi Education Institutes, Gannan Normal University, Ganzhou 341000, Jiangxi, China (e-mail: \href{xianzeqiang@gnnu.edu.cn}{xianzeqiang@gnnu.edu.cn}; \href{liucaihui@gnnu.edu.cn}{liucaihui@gnnu.edu.cn}; \href{zhang_yong@gnnu.edu.cn}{zhang\_yong@gnnu.edu.cn}; \href{wenjingqiu@gnnu.edu.cn}{wenjingqiu@gnnu.edu.cn}).}%
	\thanks{Duoqian Miao is with the Department of Computer Science and Technology, Tongji University, 201804 Shanghai, China (e-mail: \href{dqmiao@tongji.edu.cn}{dqmiao@tongji.edu.cn}).}%
	\thanks{Witold Pedrycz is with the Department of Electrical and Computer Engineering, University of Alberta, Edmonton, AB T6R 2V4, Canada (e-mail: \href{wpedrycz@ualberta.ca}{wpedrycz@ualberta.ca}).}%
	\thanks{* Caihui Liu is the corresponding author.}%
}

\markboth{}%
{Xian \MakeLowercase{\textit{et al.}}: A Boundary-Aware Non-parametric Granular-Ball Classifier Based on Minimum Description Length}
\maketitle

\begin{abstract}
	Existing granular-ball classification methods are often driven by handcrafted quality measures, neighborhood rules, or heuristic splitting and stopping criteria, which may reduce the transparency of local construction decisions and hinder explicit modeling of boundary-sensitive regions. To address this issue, this paper proposes a Minimum Description Length based Granular-Ball Classifier (MDL-GBC), a boundary-aware non-parametric and interpretable granular-ball classifier. MDL-GBC formulates class-conditional granular-ball construction as a local model selection problem under the Minimum Description Length principle. For each class, samples from the target class provide positive class evidence, while samples from the remaining classes provide negative boundary evidence. For each current granular ball, three candidate explanations are compared under a unified description-length criterion: a single-ball model, a two-ball model, and a core-boundary model. The selected model determines whether the ball is retained, geometrically split, or refined into core and boundary-sensitive child balls, thereby making local construction decisions consistent with the MDL-based classification mechanism. During prediction, a class-level mixture coding rule aggregates stable granular balls of the same class and assigns the test sample by comparing class-wise coding costs. Experiments on 18 benchmark datasets show that MDL-GBC achieves competitive classification performance against classical classifiers and representative granular-ball-based methods, obtaining the best average Accuracy, Macro-F1, and average rank. These results indicate that MDL-GBC provides an effective and interpretable alternative to conventional heuristic granular-ball classification strategies.
\end{abstract}


\section{Introduction}
\label{sec:introduction}

\IEEEPARstart{C}{lassification} is a fundamental task in machine learning and data mining.
Conventional classifiers often learn decision rules from sample-level relations, where individual samples are treated as the basic computational units~\cite{Xia2024_GBCSurvey}.
Although such fine-grained representations are flexible and widely used, they may be sensitive to local perturbations, noisy labels, and ambiguous samples near class boundaries.
In many practical scenarios, the reliability of a classification decision depends not only on isolated samples, but also on the local regions in which these samples are located.
Therefore, an effective classifier should be able to construct data-adaptive local regions that reflect distributional structure, capture boundary information, and remain interpretable in terms of decision formation.

Granular-ball computing provides a region-level learning paradigm for this purpose~\cite{Xia2024_GBCSurvey}.
Instead of using individual samples as basic units, it represents a group of similar samples by a granular ball, which is usually described by its center, radius, and label.
Xia et al.~\cite{Xia2019_GBClassifiers} introduced granular-ball classifiers by replacing sample-level inputs with granular balls, showing that this representation can reduce redundant point-wise computation and improve robustness.
Since each granular ball has an explicit geometric meaning, the learned model can also be interpreted through local information granules rather than isolated samples.
Following this idea, granular-ball learning has been extended to support vector machines~\cite{lang2025granular,xia2024gbsvm}, rough-set-based knowledge representation~\cite{xia2023gbrs}, fuzzy-set-based learning~\cite{Xia2024_GBFuzzySet}, clustering~\cite{Xie2023_GBSC,Xia2025_GBCT}, feature selection~\cite{zhang2025attribute,Zhang2025_FAGBRS}, and anomaly or outlier detection~\cite{Su2024_GBFRD,Su2025_GBDO}.
These studies indicate that granular balls can serve as efficient and interpretable carriers of multi-granularity knowledge.

The effectiveness of granular-ball learning, however, depends strongly on how granular balls are constructed.
Early granular-ball generation methods usually adopt recursive splitting strategies, where a ball is divided until a predefined quality condition, such as class purity, is satisfied~\cite{Xia2019_GBClassifiers}.
Subsequent studies have improved this process from different perspectives.
For example, Xia et al.~\cite{Xia2024_EfficientAdaptiveGBG} improved generation efficiency by replacing repeated clustering operations with a more efficient division mechanism.
Xie et al.~\cite{Xie2024_GBGPP} reduced the instability caused by center selection through data-driven center initialization and local outlier handling.
Liu et al.~\cite{Liu2025_LDGBG} introduced local density information to improve distributional consistency.
Other studies have considered heterogeneous class relationships~\cite{Liao2025_ADPGBG}, radius adjustment and overlap reduction~\cite{Pan2025_GranularityTuning}, and justifiable granularity~\cite{Jia2025_POJGGeneration}.
These methods have improved the efficiency, stability, and quality of generated granular balls.

Despite these advances, most existing generation strategies still rely on task-specific or separately designed criteria to determine local construction operations.
For instance, one rule may be used to decide whether a ball should be split, another criterion may be introduced to adjust its radius, and additional mechanisms may be required to reduce overlaps or handle uncertain samples.
As a result, ball retention, geometric splitting, boundary handling, and granularity control are often governed by different heuristic rules rather than by a unified decision principle.
This makes it difficult to determine whether a local construction decision is sufficiently supported by the data itself or mainly triggered by a predefined rule.
For classification tasks, this issue is particularly important because the learned granular balls should not only be compact within each class, but also reliable with respect to nearby samples from other classes.

Boundary information further highlights this limitation.
A granular ball may be compact with respect to samples of its own class, while still being close to samples from other classes.
In such cases, within-class compactness alone is insufficient to determine whether the ball provides a reliable class-conditional representation.
Several uncertainty-aware granular-ball methods have addressed related issues from different perspectives.
Granular-ball fuzzy set models introduce fuzzy membership into ball-based representations~\cite{Xia2024_GBFuzzySet}.
Three-way granular-ball classifiers allow uncertain samples to be handled separately instead of forcing all instances into certain classes~\cite{Yang2024_3WCGBNRS,Yang2026_ILGBSS}.
Shadowed granular-ball generation characterizes local uncertainty through shadowed-set-based three-region granulation~\cite{Zhang2026_SGB,Yang2025_SGB}.
These studies show that uncertainty and boundary information are important for granular-ball classification.
Nevertheless, boundary information is often introduced as an additional uncertainty-handling mechanism, rather than being directly embedded into the criterion that determines how granular balls are constructed.
Therefore, a classification-oriented granular-ball method should jointly consider within-class compactness, cross-class boundary exposure, and model complexity during the construction process.

The Minimum Description Length (MDL) principle provides a principled perspective for addressing this problem~\cite{barron1998minimum,grunwald2007minimum}.
Instead of relying on a fixed threshold to decide whether a local region should be refined, MDL selects the model that gives the most economical description of the data.
A more complex representation is accepted only when its improvement in data encoding is sufficient to compensate for the additional coding cost.
This trade-off is naturally consistent with granular-ball learning.
Simple and well-separated regions should remain coarse, whereas complex or boundary-sensitive regions should be refined only when such refinement is justified by the data.
Thus, MDL provides a suitable basis for data-adaptive granularity control and interpretable local construction decisions.

The relevance of MDL to granular-ball construction has also been indicated by recent work on MDL-based granular-ball generation for clustering~\cite{Xian2026_MDLGBG}.
However, classification differs from clustering in that class labels and cross-class relationships are directly involved in representation learning.
A classification-oriented granular-ball method should therefore not only describe the distribution of samples within each class, but also account for negative evidence from other classes near decision boundaries.
Accordingly, the construction criterion should evaluate each granular ball not only by its within-class representation, but also by its reliability under nearby cross-class interactions.

Motivated by the above analysis, this paper proposes MDL-GBC, a boundary-aware non-parametric granular-ball classifier based on the Minimum Description Length principle.
The proposed method constructs class-conditional granular balls by jointly considering positive evidence from the target class and negative boundary evidence from the remaining classes.
Under the MDL principle, class-conditional granular-ball construction is formulated as a local model selection problem, where ball retention, geometric splitting, and boundary-sensitive refinement are compared within a unified description-length criterion.
After training, each class is represented by a set of stable granular balls, and prediction is performed by comparing class-level mixture coding costs.
In this way, MDL-GBC provides a classification-oriented granular-ball learning mechanism that reduces dependence on manually specified structural thresholds while maintaining consistency between representation construction and prediction from a coding-theoretic perspective.

The main contributions of this paper are summarized as follows:
\begin{itemize}
	\item A boundary-aware non-parametric granular-ball classifier is proposed under the MDL principle, providing a unified coding-theoretic framework for class-conditional granular-ball construction and prediction.
	
	\item A class-conditional construction criterion is developed by incorporating both positive class evidence and negative boundary evidence, enabling the learned balls to reflect within-class compactness and cross-class boundary exposure.
	
	\item Granular-ball construction is formulated as a local model selection problem, in which ball retention, geometric splitting, and boundary-sensitive refinement are selected according to a unified description-length criterion.
	
	\item A class-level mixture coding rule is introduced for prediction, allowing stable granular balls of the same class to jointly contribute to the final classification decision.
\end{itemize}

The remainder of this paper is organized as follows.
Section~\ref{sec:methodology} presents the proposed MDL-GBC classifier in detail, including the problem formulation, class-conditional granular-ball representation, negative boundary evidence, MDL-based local model competition, and class-level mixture coding prediction rule.
Section~\ref{sec:complexity_analysis} analyzes the computational complexity of the proposed method.
Section~\ref{sec:experiments} describes the experimental settings and reports the classification results on 18 benchmark datasets.
Section~\ref{sec:discussion_conclusion} provides discussion and concluding remarks, including the empirical characteristics, advantages, limitations, and future directions of MDL-GBC.

\section{Methodology}
\label{sec:methodology}

This section presents MDL-GBC, a boundary-aware non-parametric granular-ball classifier based on the Minimum Description Length (MDL) principle.
The method follows the general idea of MDL-based local model competition used in MDL-GBG~\cite{Xian2026_MDLGBG}, but reformulates it for supervised classification.
In MDL-GBC, granular balls are constructed in a class-conditional manner.
For each target class, samples from this class provide positive evidence, whereas samples from the remaining classes provide negative boundary evidence.
Thus, local construction is no longer guided only by the compactness of the current ball, but also by its exposure to samples from other classes.

Under this setting, each current class-conditional ball is evaluated by three candidate local explanations: a single-ball model, a two-ball model, and a core-boundary model.
The selected explanation determines whether the current ball is retained as a stable ball, geometrically split into two child balls, or refined into a core ball and a boundary-sensitive child ball.
Different from the core-ball-plus-residual model used in MDL-GBG for clustering, the boundary-sensitive child ball in MDL-GBC is not treated as a residual set.
It remains part of the class-conditional representation and continues to participate in recursive MDL evaluation.
After construction, each class is represented by a set of stable granular balls, and prediction is performed by comparing class-level mixture coding costs.

\subsection{Problem Formulation and Notation}
\label{subsec:problem_formulation}

Let
$
	D
	=
	\{(x_i,y_i)\}_{i=1}^{n},
	x_i\in\mathbb{R}^{d},
	y_i\in Y=\{1,2,\ldots,C\},
	\label{eq:training_set}
$
be a labeled training set, where $n$ denotes the number of samples, $d$ denotes the feature dimension, and $C$ denotes the number of classes.
Before granular-ball construction, all features are linearly normalized into $[0,1]$.
The same normalization transformation is then applied to test samples.

For class $c\in Y$, the positive and negative sample sets are defined as
\begin{equation}
	X_c^+
	=
	\{x_i\mid y_i=c\},
	\qquad
	X_c^-
	=
	\{x_i\mid y_i\neq c\}.
	\label{eq:positive_negative_sets}
\end{equation}
Let $n_c=|X_c^+|$ and $n_c^-=|X_c^-|$.
The smoothed empirical prior of class $c$ is defined as
\begin{equation}
	\pi_c
	=
	\frac{n_c+1}{n+C}.
	\label{eq:class_prior}
\end{equation}
This prior incorporates class-frequency information into the coding objective and avoids zero-probability estimates.

\subsection{Class-Conditional Granular Ball}
\label{subsec:class_conditional_granular_ball}

\begin{definition}[Class-conditional granular ball]
	For class $c$, a class-conditional granular ball is defined as
	\begin{equation}
		B=(X_B,c),
		\qquad
		X_B\subseteq X_c^+,
		\label{eq:class_conditional_ball}
	\end{equation}
	where $X_B$ is a subset of positive samples from class $c$.
\end{definition}

For a ball $B$, let $n_B=|X_B|$.
Its empirical center is given by
\begin{equation}
	\mu_B
	=
	\frac{1}{n_B}
	\sum_{x\in X_B}x .
	\label{eq:center}
\end{equation}
The empirical radius is defined as
\begin{equation}
	r_B
	=
	\max_{x\in X_B}
	\|x-\mu_B\|_2 .
	\label{eq:radius}
\end{equation}
For the $j$-th feature, the empirical variance within $B$ is
\begin{equation}
	v_{Bj}
	=
	\frac{1}{n_B}
	\sum_{x\in X_B}x_j^2
	-
	\mu_{Bj}^{2},
	\qquad
	j=1,2,\ldots,d .
	\label{eq:empirical_variance}
\end{equation}

The first- and second-order sufficient statistics of $B$ are
\begin{equation}
	s_B
	=
	\sum_{x\in X_B}x,
	\qquad
	a_B
	=
	\sum_{x\in X_B}x\odot x,
	\label{eq:sufficient_statistics}
\end{equation}
where $\odot$ denotes the Hadamard product.
These statistics allow candidate local decompositions to be evaluated without repeatedly recomputing empirical moments.

To keep logarithmic and ratio terms well-defined, the following effective radius and variance are used:
\begin{equation}
	\tilde{r}_B
	=
	\max\{r_B,\epsilon_r\},
	\qquad
	\tilde{v}_{Bj}
	=
	\max\{v_{Bj},\epsilon_v\},
	\label{eq:effective_radius_variance}
\end{equation}
where $\epsilon_r$ and $\epsilon_v$ are small numerical constants used for finite-precision stability.

For prediction, the Gaussian coding energy is evaluated with training-data-dependent variance floors.
Let $r_0$ denote a radius floor estimated from the stable balls, and let $\eta_j$ denote a global variance floor for the $j$-th normalized feature.
The effective quantities used during prediction are
\begin{equation}
	\tilde r_B^{\mathrm{pred}}
	=
	\max\{r_B,r_0,\epsilon_r\},
	\label{eq:predictive_effective_radius}
\end{equation}
and
\begin{equation}
	\tilde v_{Bj}^{\mathrm{pred}}
	=
	\max
	\left\{
	v_{Bj},
	\eta_j,
	\frac{(\tilde r_B^{\mathrm{pred}})^2}{d},
	\epsilon_v
	\right\}.
	\label{eq:predictive_effective_variance}
\end{equation}

\subsection{Negative Boundary Evidence}
\label{subsec:negative_boundary_evidence}

For a class-conditional ball, compactness alone is insufficient for classification, because a compact positive region may still lie close to samples from other classes.
MDL-GBC therefore introduces negative boundary evidence through nearest-negative distances.
For class $c$, the nearest-negative distance of a sample $x$ is defined as
\begin{equation}
	\delta_c^{-}(x)
	=
	\min_{z\in X_c^-}
	\|x-z\|_2 .
	\label{eq:nearest_negative_distance}
\end{equation}
The nearest-negative distance of the ball center is
\begin{equation}
	\delta_c^{-}(\mu_B)
	=
	\min_{z\in X_c^-}
	\|\mu_B-z\|_2 .
	\label{eq:center_nearest_negative_distance}
\end{equation}
When $X_c^-$ is empty, the boundary-related terms are set to zero.

For $x\in X_B$, its relative boundary risk with respect to $B$ and class $c$ is defined as
\begin{equation}
	\rho(x;B,c)
	=
	\frac{1}{
		1+
		\left(
		\dfrac{\delta_c^{-}(x)}{\tilde{r}_B}
		\right)^2
	}.
	\label{eq:relative_boundary_risk}
\end{equation}

\begin{property}[Monotonicity of relative boundary risk]
	\label{prop:boundary_risk_monotonicity}
	For a fixed ball $B$, the relative boundary risk $\rho(x;B,c)$ is monotonically decreasing with respect to the normalized nearest-negative distance $\delta_c^{-}(x)/\tilde r_B$.
\end{property}

\begin{proof}
	Let
	\begin{equation}
		u
		=
		\frac{\delta_c^{-}(x)}{\tilde r_B}.
	\end{equation}
	Then $\rho=(1+u^2)^{-1}$, and
	\begin{equation}
		\frac{\partial \rho}{\partial u}
		=
		-\frac{2u}{(1+u^2)^2}
		\leq 0,
		\qquad
		u\geq 0.
		\label{eq:boundary_risk_monotonicity}
	\end{equation}
	Therefore, $\rho(x;B,c)$ decreases as the normalized nearest-negative distance increases.
\end{proof}

The average boundary risk of $B$ is
\begin{equation}
	\bar{\rho}(B,c)
	=
	\frac{1}{n_B}
	\sum_{x\in X_B}
	\rho(x;B,c).
	\label{eq:average_boundary_risk}
\end{equation}

\subsection{Description Length of a Class-Conditional Ball}
\label{subsec:description_length_ball}

For a class-conditional ball $B$ of class $c$, the total description length is defined as
\begin{equation}
	L(B,c)
	=
	L_{\mathrm{data}}(B)
	+
	L_{\mathrm{sep}}(B,c)
	+
	L_{\mathrm{cls}}(c),
	\label{eq:total_ball_length}
\end{equation}
where $L_{\mathrm{data}}(B)$ describes the local feature distribution, $L_{\mathrm{sep}}(B,c)$ measures separation from negative evidence, and $L_{\mathrm{cls}}(c)$ encodes the class assignment.

\subsubsection{Data Encoding Term}

Samples inside $B$ are modeled by a diagonal Gaussian distribution.
Using Eq.~\eqref{eq:effective_radius_variance}, the data encoding length is defined as
\begin{equation}
	L_{\mathrm{data}}(B)
	=
	\frac{n_B}{2}
	\sum_{j=1}^{d}
	\left[
	1+\ln(2\pi \tilde{v}_{Bj})
	\right]
	+
	d\ln\max(n_B,2).
	\label{eq:data_encoding_length}
\end{equation}

\noindent\textit{Derivation.}
For a diagonal Gaussian model, the local density of $x\in X_B$ is
\begin{equation}
	p(x\mid B)
	=
	\prod_{j=1}^{d}
	\frac{1}{\sqrt{2\pi v_{Bj}}}
	\exp
	\left[
	-
	\frac{(x_j-\mu_{Bj})^2}{2v_{Bj}}
	\right].
	\label{eq:diagonal_gaussian_density}
\end{equation}
The negative log-likelihood of all samples in $B$ is
\begin{align}
	-\sum_{x\in X_B}\ln p(x\mid B)
	&=
	\frac{1}{2}
	\sum_{x\in X_B}
	\sum_{j=1}^{d}
	\left[
	\ln(2\pi v_{Bj})
	+
	\frac{(x_j-\mu_{Bj})^2}{v_{Bj}}
	\right].
	\label{eq:gaussian_nll_expand}
\end{align}
Using the empirical variance identity
\begin{equation}
	\frac{1}{n_B}
	\sum_{x\in X_B}
	(x_j-\mu_{Bj})^2
	=
	v_{Bj},
	\label{eq:empirical_variance_identity}
\end{equation}
the maximum-likelihood coding term becomes
\begin{equation}
	-\sum_{x\in X_B}\ln p(x\mid B)
	=
	\frac{n_B}{2}
	\sum_{j=1}^{d}
	\left[
	1+\ln(2\pi v_{Bj})
	\right].
	\label{eq:gaussian_nll_mle}
\end{equation}
Replacing $v_{Bj}$ by its effective form $\tilde v_{Bj}$ gives the stabilized MLE-style coding surrogate used in Eq.~\eqref{eq:data_encoding_length}.
When $v_{Bj}=\tilde v_{Bj}$, this expression coincides with the empirical maximum-likelihood coding term.
The additional term $d\ln\max(n_B,2)$ penalizes the complexity of the local Gaussian description and keeps the penalty well-defined for small balls.

\subsubsection{Boundary Separation Term}

The boundary separation term consists of an intrusion penalty and a margin penalty:
\begin{equation}
	L_{\mathrm{sep}}(B,c)
	=
	L_{\mathrm{intr}}(B,c)
	+
	L_{\mathrm{mar}}(B,c).
	\label{eq:separation_length}
\end{equation}
The intrusion penalty is defined as
\begin{equation}
	L_{\mathrm{intr}}(B,c)
	=
	n_B
	\left[
	-\ln
	\max(1-\bar{\rho}(B,c),\epsilon_{\mathrm{num}})
	\right],
	\label{eq:intrusion_length}
\end{equation}
where $\epsilon_{\mathrm{num}}>0$ keeps the logarithmic term well-defined.

The margin penalty is based on the normalized overlap ratio
\begin{equation}
	\omega(B,c)
	=
	\frac{
		\max\{0,\tilde{r}_B-\delta_c^{-}(\mu_B)\}
	}{
		\tilde{r}_B
	}.
	\label{eq:normalized_overlap}
\end{equation}
The corresponding coding cost is
\begin{equation}
	L_{\mathrm{mar}}(B,c)
	=
	n_B\ln\bigl(1+\omega(B,c)\bigr).
	\label{eq:margin_length}
\end{equation}

By construction, $L_{\mathrm{intr}}(B,c)$ increases with the average boundary risk $\bar{\rho}(B,c)$ before numerical clipping, and is non-decreasing after applying the lower bound $\epsilon_{\mathrm{num}}$.
Similarly, $L_{\mathrm{mar}}(B,c)$ increases with the overlap ratio $\omega(B,c)$.
Therefore, balls that are more exposed to nearby negative evidence receive larger separation costs.

\subsubsection{Class Assignment Term}

The class assignment term is
\begin{equation}
	L_{\mathrm{cls}}(c)
	=
	-\ln\pi_c,
	\label{eq:class_prior_length}
\end{equation}
where $\pi_c$ is defined in Eq.~\eqref{eq:class_prior}.
When a region is decomposed into multiple child balls, this term is counted for each child explanation.
Thus, additional fragmentation is accepted only when the reduction in data and separation costs compensates for the increased coding cost.

\subsection{Local MDL Model Competition}
\label{subsec:local_mdl_model_competition}

For each current ball $B$ of class $c$, three candidate local models are evaluated:
\begin{itemize}
	\item $M_1$: single-ball model;
	\item $M_2$: two-ball model;
	\item $M_3$: core-boundary model.
\end{itemize}
The selected model determines whether $B$ is retained as a stable ball, geometrically split into two child balls, or decomposed into a core ball and a boundary-sensitive child ball.

To avoid degenerate binary decompositions, each class is assigned a minimum resolvable granularity.
We first define the class-wise granularity candidate
\begin{equation}
	\alpha_c
	=
	\left\lceil
	\min
	\left\{
	\dfrac{\sqrt{n_c}}{\ln(\sqrt{d+2})},
	d+2
	\right\}
	\right\rceil .
	\label{eq:adaptive_granularity_candidate}
\end{equation}
The minimum admissible child-ball size is then defined as
\begin{equation}
	n_{\min}^{(c)}
	=
	\begin{cases}
		1, & n_c\leq 3,\\[3pt]
		\min\left\{
		\left\lfloor \dfrac{n_c}{2}\right\rfloor,
		\max\{2,\alpha_c\}
		\right\},
		& n_c>3 .
	\end{cases}
	\label{eq:adaptive_n_min}
\end{equation}
Here, $\sqrt{n_c}$ reflects the available class-wise sample scale, while $d+2$ provides a dimension-related reference.
The logarithmic correction moderates the influence of dimensionality, and the truncation by $\lfloor n_c/2\rfloor$ keeps binary decompositions feasible.
Thus, the admissible granularity is determined from the class size and feature dimension, rather than from a user-specified splitting threshold.

\subsubsection{Single-Ball Model}

The single-ball model assumes that the current ball already provides an adequate local explanation.
Its description length is
\begin{equation}
	L_1(B,c)
	=
	L(B,c).
	\label{eq:single_ball_model}
\end{equation}
If $M_1$ is selected, $B$ is retained as a stable class-conditional ball.

\subsubsection{Two-Ball Model}

The two-ball model assumes that $B$ contains two distinguishable geometric substructures.
It decomposes $B$ into two disjoint child balls:
\begin{equation}
	B
	\longrightarrow
	(B_L,B_R),
	X_{B_L}\cap X_{B_R}=\varnothing,
	X_{B_L}\cup X_{B_R}=X_B .
	\label{eq:two_ball_partition}
\end{equation}
The split is admissible only when
\begin{equation}
	n_{B_L}\geq n_{\min}^{(c)},
	\qquad
	n_{B_R}\geq n_{\min}^{(c)}.
	\label{eq:admissible_two_ball_split}
\end{equation}

For an admissible binary partition $B\rightarrow(B_L,B_R)$, the membership assignment of $n_B$ samples into two child balls with sizes $n_{B_L}$ and $n_{B_R}$ has
\begin{equation}
	\binom{n_B}{n_{B_L}}
	=
	\frac{n_B!}{n_{B_L}!n_{B_R}!},
	\qquad
	n_{B_L}+n_{B_R}=n_B
	\label{eq:number_of_partitions}
\end{equation}
possible configurations.
The code length for specifying this membership pattern is
\begin{equation}
	L_{\mathrm{part}}(B_L,B_R)
	=
	\ln
	\binom{n_B}{n_{B_L}} .
	\label{eq:partition_cost_binomial}
\end{equation}

\begin{proposition}[Entropy form of the partition code]
	\label{prop:partition_entropy}
	For a binary partition $B\rightarrow(B_L,B_R)$ with
	$p_L=n_{B_L}/n_B$ and $p_R=n_{B_R}/n_B$, the membership-pattern code satisfies
	\begin{equation}
		\ln\binom{n_B}{n_{B_L}}
		=
		n_BH(p_L,p_R)
		+
		O(\ln n_B),
		\label{eq:partition_entropy_bound}
	\end{equation}
	where
	$
	H(p_L,p_R)=-p_L\ln p_L-p_R\ln p_R
	$
	is the binary entropy.
\end{proposition}

\begin{proof}
	Using Stirling's approximation
	\begin{equation}
		\ln n!
		=
		n\ln n-n+O(\ln n),
	\end{equation}
	we obtain
	\begin{align}
		\ln
		\binom{n_B}{n_{B_L}}
		&=
		\ln n_B!
		-
		\ln n_{B_L}!
		-
		\ln n_{B_R}!
		\nonumber\\
		&=
		n_B\ln n_B
		-
		n_{B_L}\ln n_{B_L}
		-
		n_{B_R}\ln n_{B_R}
		+
		O(\ln n_B).
	\end{align}
	Substituting $n_{B_L}=n_Bp_L$ and $n_{B_R}=n_Bp_R$ gives
	\begin{equation}
		\ln
		\binom{n_B}{n_{B_L}}
		=
		-n_B(p_L\ln p_L+p_R\ln p_R)
		+
		O(\ln n_B).
	\end{equation}
	Therefore,
	\begin{equation}
		\ln
		\binom{n_B}{n_{B_L}}
		=
		n_BH(p_L,p_R)+O(\ln n_B).
	\end{equation}
\end{proof}

Ignoring the lower-order term $O(\ln n_B)$ in Proposition~\ref{prop:partition_entropy}, the partition coding cost is written as
\begin{equation}
	L_{\mathrm{part}}(B_L,B_R)
	=
	n_B
	H
	\left(
	\frac{n_{B_L}}{n_B},
	\frac{n_{B_R}}{n_B}
	\right).
	\label{eq:partition_cost}
\end{equation}

In addition to the membership-pattern cost, a logarithmic candidate-selection term is used:
\begin{equation}
	L_{\mathrm{sel}}^{(2)}(B)
	=
	\ln\max(n_B,2)
	+
	\ln\max(|V_B|,1).
	\label{eq:two_ball_selection_cost}
\end{equation}
The first term accounts for the finite search over admissible ordered cuts, while the second term accounts for selecting a candidate projection direction.
This cost does not encode the membership assignment again; the induced binary membership structure is already represented by $L_{\mathrm{part}}(B_L,B_R)$.
The description length of a candidate two-ball explanation is
\begin{align}
	L_2(B_L,B_R;c)
	&=
	L_{\mathrm{part}}(B_L,B_R)
	+
	L_{\mathrm{sel}}^{(2)}(B)
	\nonumber\\
	&\quad
	+
	L(B_L,c)
	+
	L(B_R,c).
	\label{eq:two_ball_length}
\end{align}

Candidate splits are generated by projecting the samples in $B$ onto a set of data-induced directions.
Let $V_B$ denote this direction set.
It contains directions derived from the internal geometry of $B$ and, when available, directions induced by negative boundary evidence.

Let
\begin{equation}
	\Sigma_B
	=
	\frac{1}{n_B}
	\sum_{x\in X_B}
	(x-\mu_B)(x-\mu_B)^\top
	\label{eq:local_covariance}
\end{equation}
be the empirical covariance matrix of $B$.
The principal geometric direction is
\begin{equation}
	v_{\mathrm{pca}}
	=
	\arg\max_{\|v\|_2=1}
	v^\top \Sigma_B v .
	\label{eq:pca_direction}
\end{equation}
When $n_c^->0$, let $\mathcal{N}_{k_B}^{-}(\mu_B)$ be the set of $k_B$ nearest negative samples to $\mu_B$, where
\begin{equation}
	k_B
	=
	\min
	\left\{
	n_c^-,
	\left\lceil\sqrt{n_B}\right\rceil
	\right\}.
	\label{eq:k_negative}
\end{equation}
The negative-evidence direction is defined as
\begin{equation}
	v_{\mathrm{neg}}
	=
	\mu_B
	-
	\frac{1}{k_B}
	\sum_{z\in \mathcal{N}_{k_B}^{-}(\mu_B)}
	z .
	\label{eq:negative_direction}
\end{equation}

To describe the contrast between relatively safe and boundary-exposed samples, define
\begin{equation}
	h_B
	=
	\max
	\left\{
	1,
	\min
	\left(
	\left\lfloor \frac{n_B}{2}\right\rfloor,
	\left\lceil\sqrt{n_B}\right\rceil
	\right)
	\right\}.
	\label{eq:h_safe_risky}
\end{equation}
Let $S_B$ be the $h_B$ samples with the largest $\delta_c^{-}(x)$, and let $R_B$ be the $h_B$ samples with the smallest $\delta_c^{-}(x)$.
The safe-boundary contrast direction is
\begin{equation}
	v_{\mathrm{sb}}
	=
	\frac{1}{h_B}
	\sum_{x\in S_B}x
	-
	\frac{1}{h_B}
	\sum_{x\in R_B}x .
	\label{eq:safe_boundary_direction}
\end{equation}
A variance-dominant coordinate direction is also included:
\begin{equation}
	v_{\mathrm{var}}
	=
	e_{j^\ast},
	\qquad
	j^\ast
	=
	\arg\max_{1\leq j\leq d} v_{Bj}.
	\label{eq:fallback_direction}
\end{equation}
All nonzero directions are normalized to unit length, and duplicate directions are removed.

For each $v\in V_B$, samples are projected as
\begin{equation}
	t_i=x_i^\top v,
	\qquad
	x_i\in X_B.
	\label{eq:projection_value}
\end{equation}
The samples are then sorted according to $t_i$, and all admissible cut positions are examined.
Let $K_{B,v}^{(2)}$ denote the set of valid cut positions along direction $v$.
The optimal two-ball description length is
\begin{equation}
	L_2^\ast(B,c)
	=
	\min_{v\in V_B}
	\min_{k\in K_{B,v}^{(2)}}
	L_2
	\left(
	B_L^{(v,k)},
	B_R^{(v,k)};
	c
	\right).
	\label{eq:optimal_two_ball_length}
\end{equation}

\subsubsection{Core-Boundary Model}

The core-boundary model describes boundary heterogeneity within a local region.
A ball may be geometrically compact while still containing samples with different degrees of exposure to negative evidence.
This model therefore decomposes $B$ into an interior core and a boundary-sensitive subset:
\begin{equation}
	B
	\longrightarrow
	(B_{\mathrm{core}},B_{\mathrm{bd}}),
	\label{eq:core_boundary_partition}
\end{equation}
where $B_{\mathrm{core}}$ contains samples relatively far from negative evidence, and $B_{\mathrm{bd}}$ contains samples closer to negative evidence.

Let the samples in $B$ be sorted by descending nearest-negative distance:
\begin{equation}
	\delta_c^{-}(x_{(1)})
	\geq
	\delta_c^{-}(x_{(2)})
	\geq
	\cdots
	\geq
	\delta_c^{-}(x_{(n_B)}).
	\label{eq:boundary_sorting}
\end{equation}
For a boundary size $n_{\mathrm{bd}}$, define
\begin{equation}
	n_{\mathrm{core}}
	=
	n_B-n_{\mathrm{bd}} .
	\label{eq:core_size}
\end{equation}
The corresponding candidate subsets are
\begin{equation}
	X_{B_{\mathrm{core}}}
	=
	\{x_{(i)}\}_{i=1}^{n_{\mathrm{core}}},
	\qquad
	X_{B_{\mathrm{bd}}}
	=
	\{x_{(i)}\}_{i=n_{\mathrm{core}}+1}^{n_B}.
	\label{eq:core_boundary_candidate}
\end{equation}
The decomposition is admissible only if
\begin{equation}
	n_{\mathrm{core}}\geq n_{\min}^{(c)},
	\qquad
	n_{\mathrm{bd}}\geq n_{\min}^{(c)}.
	\label{eq:admissible_core_boundary_split}
\end{equation}

The core-boundary decomposition also induces a binary membership pattern.
By Proposition~\ref{prop:partition_entropy}, its membership coding cost is
\begin{equation}
	L_{\mathrm{part}}
	(B_{\mathrm{core}},B_{\mathrm{bd}})
	=
	n_B
	H
	\left(
	\frac{n_{\mathrm{core}}}{n_B},
	\frac{n_{\mathrm{bd}}}{n_B}
	\right),
	\label{eq:core_boundary_partition_cost}
\end{equation}
where $H(\cdot,\cdot)$ is the binary entropy.

In addition, a logarithmic candidate-selection cost is used for the admissible boundary-size search:
\begin{equation}
	L_{\mathrm{sel}}^{(3)}(B)
	=
	\ln\max(n_B,2).
	\label{eq:core_boundary_selection_cost}
\end{equation}
This term accounts for selecting a boundary size from the ordered nearest-negative sequence.
It is separate from the membership-pattern cost in Eq.~\eqref{eq:core_boundary_partition_cost}, which encodes the induced core-boundary assignment.

The description length of a candidate core-boundary explanation is
\begin{align}
	L_3(B_{\mathrm{core}},B_{\mathrm{bd}};c)
	&=
	L_{\mathrm{part}}
	(B_{\mathrm{core}},B_{\mathrm{bd}})
	+
	L_{\mathrm{sel}}^{(3)}(B)
	\nonumber\\
	&\quad
	+
	L(B_{\mathrm{core}},c)
	+
	L(B_{\mathrm{bd}},c).
	\label{eq:core_boundary_length}
\end{align}
Let $K_B^{(3)}$ denote the set of admissible boundary sizes.
The optimal core-boundary description length is
\begin{equation}
	L_3^\ast(B,c)
	=
	\min_{n_{\mathrm{bd}}\in K_B^{(3)}}
	L_3
	\left(
	B_{\mathrm{core}}^{(n_{\mathrm{bd}})},
	B_{\mathrm{bd}}^{(n_{\mathrm{bd}})};
	c
	\right).
	\label{eq:optimal_core_boundary_length}
\end{equation}

\begin{remark}[Relation between $M_2$ and $M_3$]
	The two-ball model $M_2$ captures geometric multi-modality through projection-based splitting.
	The core-boundary model $M_3$ captures boundary heterogeneity through nearest-negative distances.
	Thus, the two models describe different local structures and may lead to different decompositions.
\end{remark}

\begin{remark}[Boundary-sensitive child ball]
	The boundary subset $B_{\mathrm{bd}}$ is not discarded and is not treated as a residual set.
	It remains a class-conditional child ball of class $c$ and is returned to the unresolved queue for further MDL evaluation.
	Consequently, all training samples are eventually assigned to stable granular balls, and no independent residual points are produced.
\end{remark}

\subsection{MDL Selection Rule}
\label{subsec:mdl_selection_rule}

For a current ball $B$ of class $c$, the competing description lengths are
\begin{equation}
	L_1(B,c),
	\qquad
	L_2^\ast(B,c),
	\qquad
	L_3^\ast(B,c).
	\label{eq:competing_lengths}
\end{equation}
For infeasible candidate decompositions, the corresponding description length is set to $+\infty$.
The best local explanation is first obtained by
\begin{equation}
	\widehat{M}
	=
	\mathop{\arg\min}\limits_{M\in\{M_1,M_2,M_3\}}
	L_M(B,c),
	\label{eq:raw_mdl_selection}
\end{equation}
where $L_{M_1}(B,c)=L_1(B,c)$, $L_{M_2}(B,c)=L_2^\ast(B,c)$, and $L_{M_3}(B,c)=L_3^\ast(B,c)$.
To avoid decisions caused by negligible numerical differences, the selected model is
\begin{equation}
	M^\ast
	=
	\begin{cases}
		\widehat{M},
		&
		\widehat{M}\neq M_1
		\ \text{and}\
		L_{\widehat{M}}(B,c)<L_1(B,c)-\epsilon_{\mathrm{mdl}},
		\\[4pt]
		M_1,
		&
		\text{otherwise}.
	\end{cases}
	\label{eq:conservative_selection_rule}
\end{equation}
If $M^\ast=M_1$, the current ball is declared stable.
If $M^\ast=M_2$ or $M^\ast=M_3$, the corresponding optimal decomposition is applied, and the two child balls are returned to the unresolved set.

\subsection{Prediction by Class-Level Mixture Coding}
\label{subsec:prediction}

After training, each class $c$ is represented by a set of stable granular balls:
\begin{equation}
	\mathcal{G}_c
	=
	\{B_{c1},B_{c2},\ldots,B_{cm_c}\},
	\qquad
	m_c=|\mathcal{G}_c|.
	\label{eq:class_ball_set}
\end{equation}
For $B_{ck}\in\mathcal{G}_c$, its mixture weight is
\begin{equation}
	w_{ck}
	=
	\frac{n_{ck}}{n_c},
	\qquad
	\sum_{k=1}^{m_c}w_{ck}=1.
	\label{eq:mixture_weight}
\end{equation}

Since one class may be represented by multiple stable balls, prediction is performed at the class level.
For a test sample $x$, its coding energy with respect to ball $B\in\mathcal{G}_c$ is
\begin{equation}
	E_c(x,B)
	=
	E_{\mathrm{g}}(x,B)
	+
	E_{\mathrm{b}}(B,c)
	+
	E_{\mathrm{o}}(x,B),
	\label{eq:test_energy}
\end{equation}
where $E_{\mathrm{g}}$ is the Gaussian feature energy, $E_{\mathrm{b}}$ is the boundary-risk energy, and $E_{\mathrm{o}}$ is the outside-ball penalty.

The Gaussian feature energy is
\begin{equation}
	E_{\mathrm{g}}(x,B)
	=
	\frac{1}{2}
	\left[
	\sum_{j=1}^{d}
	\frac{(x_j-\mu_{Bj})^2}{\tilde{v}_{Bj}^{\mathrm{pred}}}
	+
	\sum_{j=1}^{d}
	\ln(2\pi\tilde{v}_{Bj}^{\mathrm{pred}})
	\right].
	\label{eq:gaussian_energy}
\end{equation}
The boundary-risk energy is
\begin{equation}
	E_{\mathrm{b}}(B,c)
	=
	-\ln
	\max
	\left(
	1-\bar{\rho}(B,c),
	\epsilon_{\mathrm{num}}
	\right).
	\label{eq:boundary_energy}
\end{equation}
The outside-ball penalty is
\begin{equation}
	E_{\mathrm{o}}(x,B)
	=
	\ln
	\left[
	1+
	\frac{
		\max\{0,\|x-\mu_B\|_2-\tilde r_B^{\mathrm{pred}}\}
	}{
		\tilde r_B^{\mathrm{pred}}
	}
	\right].
	\label{eq:outside_penalty}
\end{equation}

The local energies are aggregated into a class-level coding likelihood:
\begin{equation}
	\mathcal{A}_c(x)
	=
	\sum_{k=1}^{m_c}
	w_{ck}
	\exp
	\left(
	-E_c(x,B_{ck})
	\right).
	\label{eq:class_coding_likelihood}
\end{equation}
The prior-weighted coding likelihood of class $c$ is
\begin{equation}
	\mathcal{J}_c(x)
	=
	\pi_c \mathcal{A}_c(x).
	\label{eq:prior_weighted_coding_likelihood}
\end{equation}
The quantities $\mathcal{A}_c(x)$ and $\mathcal{J}_c(x)$ are used for relative coding comparison among classes and are not required to be normalized probability densities over the input space.
Taking the negative logarithm gives the class-level mixture coding cost:
\begin{align}
	S_c(x)
	&=
	-\ln \mathcal{J}_c(x)
	\nonumber\\
	&=
	-\ln\pi_c
	-
	\ln
	\left[
	\sum_{k=1}^{m_c}
	w_{ck}
	\exp
	\left(
	-E_c(x,B_{ck})
	\right)
	\right].
	\label{eq:class_score}
\end{align}
The predicted class is
\begin{equation}
	\hat{y}
	=
	\arg\min_{c\in Y}
	S_c(x).
	\label{eq:prediction_rule}
\end{equation}

\begin{remark}[Class-level decision]
	Equation~\eqref{eq:prediction_rule} compares class-level mixture coding costs.
	Thus, prediction is not based on a single nearest ball or on the minimum individual-ball cost.
	All stable balls of the same class jointly contribute to the coding score of that class.
\end{remark}

\subsection{Overall Classifier}
\label{subsec:overall_classifier}

Fig.~\ref{fig:framework} gives an overview of MDL-GBC.
The method first constructs class-conditional positive and negative evidence in a one-vs-rest manner.
Then, each current ball is evaluated under a local MDL model competition among the single-ball, two-ball, and core-boundary models.
The selected model determines whether the ball is retained as a stable ball, geometrically split, or decomposed into a core ball and a boundary-sensitive child ball for further recursive evaluation.
The resulting stable balls form the final class-wise representation and are used for class-level mixture coding during prediction.

\begin{figure*}[!t]
	\centering
	\includegraphics[width=\textwidth]{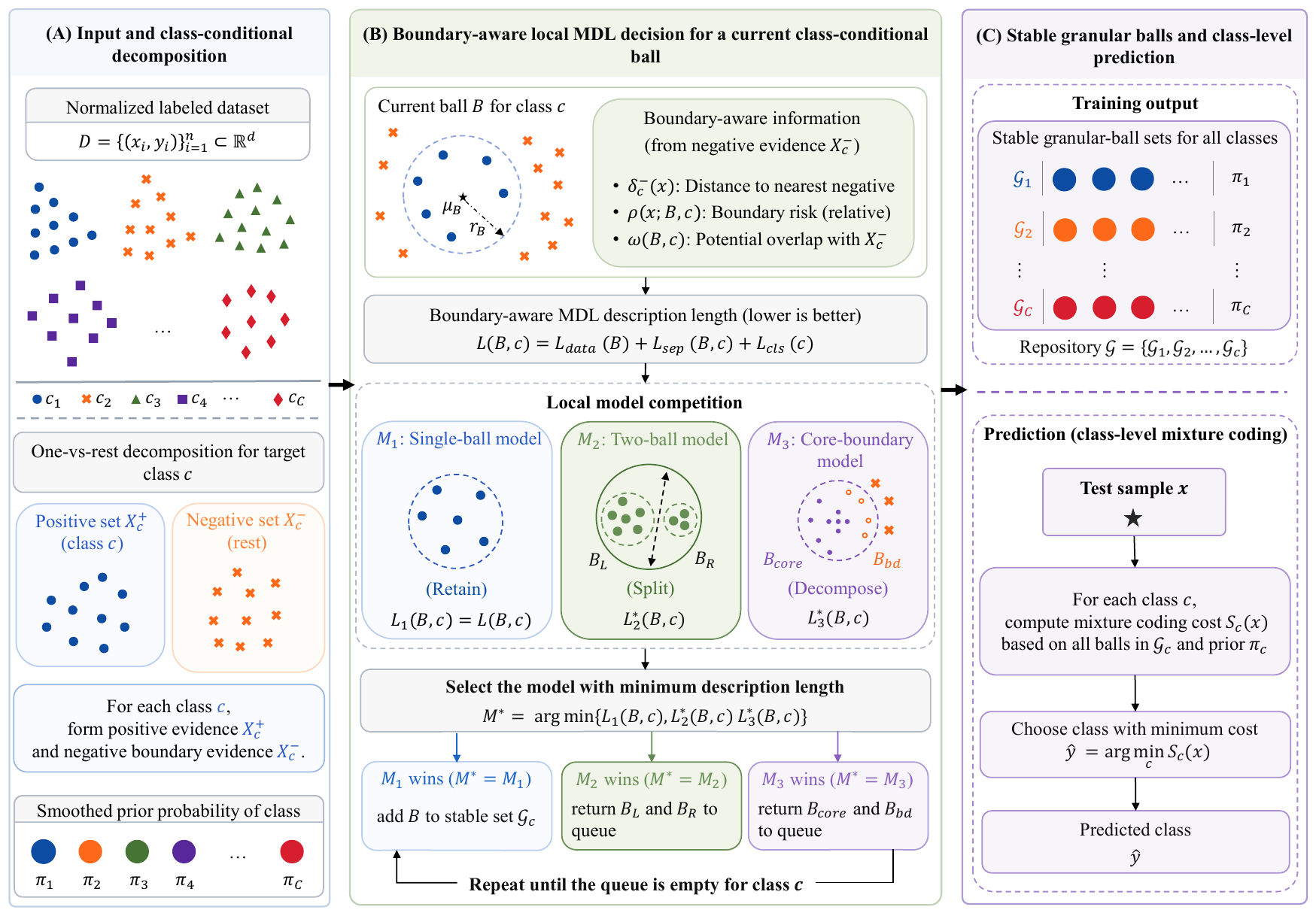}
	\caption{
		Overview of the proposed MDL-GBC framework.
		(A) The normalized labeled dataset is decomposed in a one-vs-rest manner for each target class.
		For class $c$, samples from the target class form the positive evidence $X_c^+$, while samples from the remaining classes form the negative boundary evidence $X_c^-$; a smoothed class prior $\pi_c$ is also estimated.
		(B) For a current class-conditional ball $B$, boundary-aware information induced by nearby negative evidence is incorporated into a unified MDL description length.
		Three candidate local explanations are then compared, including the single-ball model, the two-ball model, and the core-boundary model.
		The model with the minimum description length determines whether $B$ is retained as a stable ball, split into two sub-balls, or decomposed into a core ball and a boundary-sensitive child ball.
		This process is recursively repeated until no unresolved ball remains for class $c$.
		(C) The stable granular balls generated for all classes constitute the training output.
		During prediction, each class is evaluated by a class-level mixture coding cost over its stable granular-ball set and prior, and the test sample is assigned to the class with the minimum coding cost.
	}
	\label{fig:framework}
\end{figure*}

The complete procedure is summarized in Algorithm~\ref{alg:classifier}.

\begin{algorithm}[htbp]
	\caption{MDL-GBC}
	\label{alg:classifier}
	\KwIn{Training set $D=\{(x_i,y_i)\}_{i=1}^{n}$; test set $X_{\mathrm{test}}$}
	\KwOut{Predicted label set $\hat{Y}_{\mathrm{test}}$}
	
	Normalize features into $[0,1]^d$\;
	Encode class labels as $Y=\{1,2,\ldots,C\}$\;
	Estimate class priors $\{\pi_c\}_{c=1}^{C}$ by Eq.~\eqref{eq:class_prior}\;
	
	\For{$c=1$ \KwTo $C$}{
		Construct $X_c^+$ and $X_c^-$\;
		Build nearest-negative evidence from $X_c^-$\;
		Compute $n_{\min}^{(c)}$ by Eq.~\eqref{eq:adaptive_n_min}\;
		
		Initialize the unresolved queue $\mathcal{Q}$ with $B_0=(X_c^+,c)$\;
		Initialize $\mathcal{G}_c=\varnothing$\;
		
		\While{$\mathcal{Q}$ is not empty}{
			Pop a ball $B$ from $\mathcal{Q}$\;
			
			\If{$n_B<2n_{\min}^{(c)}$}{
				Record the boundary quantities of $B$\;
				Add $B$ to $\mathcal{G}_c$\;
				\textbf{continue}\;
			}
			
			Evaluate $L_1(B,c)$, $L_2^\ast(B,c)$, and $L_3^\ast(B,c)$\;
			Select $M^\ast$ according to Eq.~\eqref{eq:conservative_selection_rule}\;
			
			\uIf{$M^\ast=M_1$}{
				Record the boundary quantities of $B$\;
				Add $B$ to $\mathcal{G}_c$\;
			}
			\uElseIf{$M^\ast=M_2$}{
				Obtain the optimal two-ball decomposition $(B_L,B_R)$\;
				Insert $B_L$ and $B_R$ into $\mathcal{Q}$\;
			}
			\ElseIf{$M^\ast=M_3$}{
				Obtain the optimal core-boundary decomposition $(B_{\mathrm{core}},B_{\mathrm{bd}})$\;
				Insert $B_{\mathrm{core}}$ and $B_{\mathrm{bd}}$ into $\mathcal{Q}$\;
			}
		}
	}
	
	Set $\mathcal{G}=\bigcup_{c=1}^{C}\mathcal{G}_c$\;
	Estimate prediction-stage radius and variance floors\;
	Initialize $\hat{Y}_{\mathrm{test}}=\varnothing$\;
	
	\ForEach{$x\in X_{\mathrm{test}}$}{
		Normalize $x$ using the training transformation\;
		Compute $S_c(x)$ for all $c\in Y$\;
		Assign $\hat{y}=\arg\min_{c\in Y}S_c(x)$\;
		Append $\hat{y}$ to $\hat{Y}_{\mathrm{test}}$\;
	}
	Return $\hat{Y}_{\mathrm{test}}$\;
\end{algorithm}

\section{Complexity Analysis}
\label{sec:complexity_analysis}

This section analyzes the computational complexity of MDL-GBC. 
Since the number and sizes of generated granular balls depend on the sequence of MDL selection decisions, the training complexity is expressed in an output-sensitive form.

Feature normalization requires $O(nd)$ time and $O(nd)$ storage when the normalized data matrix is retained. 
Consider a class-conditional ball $B$ containing $m$ samples in $d$ dimensions. 
Computing its center, variance vector, sufficient statistics, and radius requires $O(md)$ time.

Let $T_{\mathrm{NN}}^{\mathrm{s}}(m,n_c^-,d)$ denote the cost of querying nearest-negative distances for $m$ samples from the negative set $X_c^-$. 
Let $T_{\mathrm{NN}}^{\mathrm{ctr}}(r,n_c^-,d)$ denote the cost of querying nearest-negative distances for $r$ candidate centers. 
Their exact forms depend on the adopted nearest-neighbor implementation.

For the single-ball model, the dominant cost comes from local statistic computation and boundary-term evaluation:
\begin{equation}
	\begin{aligned}
		T_{M_1}(B)
		=
		O\bigl(
		& md
		+
		T_{\mathrm{NN}}^{\mathrm{s}}(m,n_c^-,d)
		\\
		&+
		T_{\mathrm{NN}}^{\mathrm{ctr}}(1,n_c^-,d)
		\bigr).
	\end{aligned}
	\label{eq:m1_complexity}
\end{equation}

For the two-ball model, let $p_B=|V_B|$ be the number of candidate projection directions, and let $T_{\mathrm{pca}}(m,d)$ denote the cost of extracting the first principal direction of $B$. 
Let $K_{B,v}^{(2)}$ be the set of valid cut positions along direction $v$. 
For worst-case notation, define
\begin{equation}
	K_B^{(2)}
	=
	\max_{v\in V_B}
	|K_{B,v}^{(2)}|.
	\label{eq:max_valid_cuts_m2}
\end{equation}
For compactness, define
\begin{equation}
	\begin{aligned}
		T_{\mathrm{ctr}}^{(2)}(B)
		=
		T_{\mathrm{NN}}^{\mathrm{ctr}}
		\bigl(
		2K_B^{(2)},
		n_c^-,
		d
		\bigr),
	\end{aligned}
	\label{eq:center_query_cost_m2}
\end{equation}
and
\begin{equation}
	\Phi_2(B)
	=
	md
	+
	m\ln m
	+
	m^2d
	+
	T_{\mathrm{ctr}}^{(2)}(B).
	\label{eq:phi_m2}
\end{equation}
Then the worst-case cost of evaluating $M_2$ is
\begin{equation}
	\begin{aligned}
		T_{M_2}(B)
		=
		O\bigl(
		& T_{\mathrm{pca}}(m,d)
		+
		T_{\mathrm{NN}}^{\mathrm{s}}(m,n_c^-,d)
		\\
		&+
		p_B\Phi_2(B)
		\bigr).
	\end{aligned}
	\label{eq:m2_complexity}
\end{equation}
In $\Phi_2(B)$, the terms $md$ and $m\ln m$ correspond to projection and sorting along one candidate direction. 
The term $m^2d$ comes from exact radius evaluation over admissible cuts, and $T_{\mathrm{ctr}}^{(2)}(B)$ accounts for nearest-negative queries of candidate child centers.

For the core-boundary model, samples are sorted once according to nearest-negative distances. 
Unlike $M_2$, this model does not enumerate projection directions. 
Define
\begin{equation}
	\begin{aligned}
		T_{\mathrm{ctr}}^{(3)}(B)
		=
		T_{\mathrm{NN}}^{\mathrm{ctr}}
		\bigl(
		2|K_B^{(3)}|,
		n_c^-,
		d
		\bigr),
	\end{aligned}
	\label{eq:center_query_cost_m3}
\end{equation}
and
\begin{equation}
	\Phi_3(B)
	=
	m\ln m
	+
	m^2d
	+
	T_{\mathrm{ctr}}^{(3)}(B).
	\label{eq:phi_m3}
\end{equation}
The worst-case cost of evaluating $M_3$ is
\begin{equation}
	\begin{aligned}
		T_{M_3}(B)
		=
		O\bigl(
		& T_{\mathrm{NN}}^{\mathrm{s}}(m,n_c^-,d)
		+
		\Phi_3(B)
		\bigr).
	\end{aligned}
	\label{eq:m3_complexity}
\end{equation}
The absence of the factor $p_B$ indicates that $M_3$ uses boundary-distance ordering rather than projection-direction enumeration.

The per-ball evaluation cost is therefore
\begin{equation}
	\begin{aligned}
		T_{\mathrm{eval}}(B)
		=
		O\bigl(
		& T_{M_1}(B)
		+
		T_{M_2}(B)
		+
		T_{M_3}(B)
		\bigr),
	\end{aligned}
	\label{eq:per_ball_complexity}
\end{equation}
where infeasible candidate models are omitted.

Let $\mathcal{T}_c$ denote the set of all balls evaluated during recursive generation for class $c$. 
The total training complexity is
\begin{equation}
	\begin{aligned}
		T_{\mathrm{train}}
		=
		O\biggl(
		nd
		+
		\sum_{c=1}^{C}
		\sum_{B\in \mathcal{T}_c}
		T_{\mathrm{eval}}(B)
		\biggr).
	\end{aligned}
	\label{eq:total_training_complexity}
\end{equation}
This bound is output-sensitive because both the number of evaluated balls and their sizes are determined by the MDL selection process. 
Simple and well-separated regions may terminate early, whereas complex or boundary-sensitive regions may require further recursive evaluation.

For prediction, let
\begin{equation}
	M
	=
	\sum_{c=1}^{C}
	|\mathcal{G}_c|
	\label{eq:number_stable_balls}
\end{equation}
be the total number of stable balls. 
Computing the coding energy between one test sample and one stable ball costs $O(d)$. 
Thus, the prediction complexity for $n_{\mathrm{test}}$ test samples is
\begin{equation}
	T_{\mathrm{pred}}
	=
	O(n_{\mathrm{test}}Md).
	\label{eq:prediction_complexity}
\end{equation}

The main storage components include the normalized training data, nearest-negative structures, active balls during generation, and the final stable-ball representation. 
The normalized data require $O(nd)$ space. 
When nearest-negative structures are constructed class by class, their auxiliary storage is bounded by $O(nd)$.

The deployable statistical model stores the center, variance, radius, label, sample count, and boundary quantities of each stable ball, requiring
$
	O(Md)
$
space, excluding normalization parameters. 
If sample memberships of stable balls are also retained for analysis or visualization, the total storage becomes
$
	O(nd+Md).
$

\section{Experiments}
\label{sec:experiments}

\subsection{Experimental Settings}
\label{subsec:experimental_settings}

Experiments were conducted on 18 real UCI benchmark datasets\footnote{\url{https://archive.ics.uci.edu/}}, covering different sample sizes, feature dimensions, class numbers, and imbalance levels. Their statistics are summarized in Table~\ref{tab:dataset_info}. A dataset is marked as imbalanced when the imbalance ratio, defined as the ratio between the largest and smallest class sizes, satisfies $\mathrm{IR}\geq 3$.

\begin{table*}[!t]
	\centering
	\caption{Dataset information and abbreviations.}
	\label{tab:dataset_info}
	\scriptsize
	\setlength{\tabcolsep}{3.5pt}
	\renewcommand{\arraystretch}{0.95}
	\begin{minipage}[t]{0.43\textwidth}
		\centering
		\begin{tabular}{llrrrc}
			\toprule
			Abbr. & Dataset & Samples & Features & Classes & Type \\
			\midrule
			Lenses & Lenses & 24 & 3 & 3 & I \\
			Zoo & Zoo & 101 & 16 & 7 & I \\
			Iris & iris & 150 & 4 & 3 & B \\
			Wine & Wine & 178 & 13 & 3 & B \\
			Seeds & seeds & 210 & 8 & 3 & B \\
			Thyroid & new-thyroid & 214 & 5 & 3 & I \\
			Heart & Statlog (Heart) & 270 & 13 & 2 & B \\
			Ionosphere & Ionosphere & 351 & 34 & 2 & B \\
			Libras & libras & 360 & 90 & 15 & B \\
			\bottomrule
		\end{tabular}
	\end{minipage}
	\hfill
	\begin{minipage}[t]{0.56\textwidth}
		\centering
		\begin{tabular}{llrrrc}
			\toprule
			Abbr. & Dataset & Samples & Features & Classes & Type \\
			\midrule
			Balance & Balance Scale & 625 & 4 & 3 & I \\
			Australian & Statlog (Australian Credit Approval) & 690 & 14 & 2 & B \\
			Banknote & Banknote Authentication & 1372 & 4 & 2 & B \\
			Rice & Rice (Cammeo and Osmancik) & 3810 & 7 & 2 & B \\
			OptDigits & Optical Recognition of Handwritten Digits & 5620 & 64 & 10 & B \\
			ISOLET & isolet & 7797 & 617 & 26 & B \\
			HTRU2 & HTRU2 & 17898 & 8 & 2 & I \\
			Shuttle & Statlog (Shuttle) & 58000 & 7 & 7 & I \\
			Skin & Skin Segmentation & 245057 & 4 & 2 & I \\
			\bottomrule
		\end{tabular}
	\end{minipage}
	\vspace{0.1cm}
	
	\begin{minipage}{0.98\textwidth}
		\footnotesize
		\textbf{Note:} "B" and "I" indicate balanced and imbalanced datasets, respectively.
	\end{minipage}
\end{table*}

All datasets were preprocessed by Min--Max normalization. For each feature,
\begin{equation}
	x'=\frac{x-x_{\min}}{x_{\max}-x_{\min}},
\end{equation}
where $x_{\min}$ and $x_{\max}$ were estimated from the training fold and then applied to the corresponding test fold to avoid information leakage.

A stratified 10-fold cross-validation protocol was used for each dataset, and the results are reported as mean $\pm$ standard deviation. The main metrics are Accuracy (Acc.) and Macro-F1 (MF1), where MF1 is retained as the class-balanced metric because it reflects macro-level class-wise performance and keeps the comparison table compact.

The compared methods include four classical baselines, XGBoost~\cite{chen2016xgboost}, KNN~\cite{cover1967nearest}, CART~\cite{breiman2017classification}, and SVM~\cite{platt1999probabilistic}, and three granular-ball-related methods, GBTSVM~\cite{quadir2024granular}, ScOrGBC~\cite{guo2026scorgbc}, and SGBSkNN~\cite{Zhang2026_SGB}. XGBoost was implemented through its scikit-learn-compatible interface, and KNN, CART, and SVM were implemented using scikit-learn~\cite{pedregosa2011scikit} with default settings. The granular-ball-related baselines used the official open-source codes and recommended parameter settings. MDL-GBC is non-parametric and requires no dataset-specific hyperparameter tuning. For reproducibility, the random seed was fixed to 2035 whenever randomness was involved.

All experiments were conducted in Python on a workstation with an AMD Ryzen 9 8940HX CPU, 32 GB RAM, and Windows 11.

\subsection{Overall Classification Performance}
\label{subsec:overall_classification_performance}

\begin{table*}[htbp]
	\centering
	\caption{Dataset-level comparison of Accuracy and Macro-F1 on the selected 18 datasets.}
	\label{tab:selected_18_classification_comparison}
	\begin{threeparttable}
		\scriptsize
		\resizebox{\textwidth}{!}{
			\begin{tabular}{llcccccccc}
				\toprule
				\multirow{2}{*}{Dataset} & \multirow{2}{*}{Metric} 
				& \multicolumn{1}{c}{Proposed} 
				& \multicolumn{4}{c}{Classical Baselines} 
				& \multicolumn{3}{c}{Granular-Ball Methods} \\
				\cmidrule(lr){3-3}
				\cmidrule(lr){4-7}
				\cmidrule(lr){8-10}
				& & MDL-GBC & XGBoost & KNN & CART & SVM & GBTSVM & ScOrGBC & SGBSkNN \\
				\midrule
				
				\multirow{2}{*}{Lenses}
				& Acc. & \textbf{0.9000 $\pm$ 0.1528} & 0.8500 $\pm$ 0.1893 & 0.8500 $\pm$ 0.1893 & \textbf{0.9000 $\pm$ 0.1528} & 0.8333 $\pm$ 0.2108 & 0.8667 $\pm$ 0.2211 & 0.6667 $\pm$ 0.2357 & 0.6167 $\pm$ 0.1500 \\
				& MF1  & \textbf{0.8622 $\pm$ 0.2188} & 0.7956 $\pm$ 0.2636 & 0.8222 $\pm$ 0.2341 & \textbf{0.8622 $\pm$ 0.2188} & 0.7622 $\pm$ 0.2968 & 0.8400 $\pm$ 0.2568 & 0.5567 $\pm$ 0.3127 & 0.4267 $\pm$ 0.1937 \\
				\midrule
				
				\multirow{2}{*}{Zoo}
				& Acc. & \textbf{0.9509 $\pm$ 0.0492} & \textbf{0.9509 $\pm$ 0.0492} & \textbf{0.9509 $\pm$ 0.0492} & 0.9500 $\pm$ 0.0500 & 0.9218 $\pm$ 0.0565 & 0.9409 $\pm$ 0.0483 & 0.8909 $\pm$ 0.0702 & 0.9009 $\pm$ 0.0633 \\
				& MF1  & 0.8885 $\pm$ 0.1141 & 0.8714 $\pm$ 0.1325 & \textbf{0.8911 $\pm$ 0.1091} & 0.8651 $\pm$ 0.1382 & 0.8070 $\pm$ 0.1400 & 0.8451 $\pm$ 0.1316 & 0.7447 $\pm$ 0.1601 & 0.7853 $\pm$ 0.1304 \\
				\midrule
				
				\multirow{2}{*}{Iris}
				& Acc. & \textbf{0.9733 $\pm$ 0.0442} & 0.9400 $\pm$ 0.0629 & 0.9533 $\pm$ 0.0521 & 0.9467 $\pm$ 0.0653 & 0.9600 $\pm$ 0.0533 & 0.9667 $\pm$ 0.0447 & 0.9333 $\pm$ 0.0667 & 0.9533 $\pm$ 0.0521 \\
				& MF1  & \textbf{0.9726 $\pm$ 0.0457} & 0.9390 $\pm$ 0.0640 & 0.9520 $\pm$ 0.0540 & 0.9457 $\pm$ 0.0665 & 0.9588 $\pm$ 0.0554 & 0.9665 $\pm$ 0.0449 & 0.9318 $\pm$ 0.0682 & 0.9520 $\pm$ 0.0540 \\
				\midrule
				
				\multirow{2}{*}{Wine}
				& Acc. & 0.9667 $\pm$ 0.0444 & 0.9663 $\pm$ 0.0275 & 0.9608 $\pm$ 0.0435 & 0.8817 $\pm$ 0.0472 & \textbf{0.9889 $\pm$ 0.0333} & 0.9497 $\pm$ 0.0300 & 0.9605 $\pm$ 0.0369 & 0.9663 $\pm$ 0.0371 \\
				& MF1  & 0.9666 $\pm$ 0.0450 & 0.9651 $\pm$ 0.0291 & 0.9619 $\pm$ 0.0436 & 0.8832 $\pm$ 0.0470 & \textbf{0.9886 $\pm$ 0.0343} & 0.9518 $\pm$ 0.0282 & 0.9630 $\pm$ 0.0335 & 0.9677 $\pm$ 0.0363 \\
				\midrule
				
				\multirow{2}{*}{Seeds}
				& Acc. & \textbf{0.9381 $\pm$ 0.0565} & 0.9286 $\pm$ 0.0714 & 0.9238 $\pm$ 0.0680 & 0.9190 $\pm$ 0.0479 & 0.9286 $\pm$ 0.0532 & 0.9238 $\pm$ 0.0610 & 0.9095 $\pm$ 0.0581 & 0.9190 $\pm$ 0.0565 \\
				& MF1  & \textbf{0.9380 $\pm$ 0.0558} & 0.9280 $\pm$ 0.0705 & 0.9238 $\pm$ 0.0670 & 0.9180 $\pm$ 0.0480 & 0.9283 $\pm$ 0.0525 & 0.9238 $\pm$ 0.0604 & 0.9092 $\pm$ 0.0573 & 0.9186 $\pm$ 0.0557 \\
				\midrule
				
				\multirow{2}{*}{Thyroid}
				& Acc. & 0.9483 $\pm$ 0.0396 & \textbf{0.9582 $\pm$ 0.0480} & 0.9305 $\pm$ 0.0427 & 0.9346 $\pm$ 0.0524 & 0.9394 $\pm$ 0.0419 & 0.9210 $\pm$ 0.0362 & 0.8877 $\pm$ 0.0638 & 0.9396 $\pm$ 0.0420 \\
				& MF1  & 0.9190 $\pm$ 0.0597 & \textbf{0.9338 $\pm$ 0.0738} & 0.8972 $\pm$ 0.0637 & 0.9139 $\pm$ 0.0723 & 0.9039 $\pm$ 0.0660 & 0.8709 $\pm$ 0.0659 & 0.8585 $\pm$ 0.0762 & 0.9103 $\pm$ 0.0612 \\
				\midrule
				
				\multirow{2}{*}{Heart}
				& Acc. & 0.8000 $\pm$ 0.0744 & 0.7963 $\pm$ 0.0727 & 0.7963 $\pm$ 0.0799 & 0.7111 $\pm$ 0.0569 & 0.8185 $\pm$ 0.0692 & 0.7963 $\pm$ 0.0799 & 0.7630 $\pm$ 0.0687 & \textbf{0.8259 $\pm$ 0.0621} \\
				& MF1  & 0.7944 $\pm$ 0.0761 & 0.7891 $\pm$ 0.0761 & 0.7891 $\pm$ 0.0858 & 0.7016 $\pm$ 0.0604 & 0.8094 $\pm$ 0.0778 & 0.7908 $\pm$ 0.0807 & 0.7572 $\pm$ 0.0698 & \textbf{0.8219 $\pm$ 0.0637} \\
				\midrule
				
				\multirow{2}{*}{Ionosphere}
				& Acc. & 0.9202 $\pm$ 0.0458 & \textbf{0.9315 $\pm$ 0.0367} & 0.8461 $\pm$ 0.0575 & 0.8889 $\pm$ 0.0432 & 0.9287 $\pm$ 0.0410 & 0.8090 $\pm$ 0.0387 & 0.8575 $\pm$ 0.0529 & 0.6867 $\pm$ 0.0306 \\
				& MF1  & 0.9089 $\pm$ 0.0534 & \textbf{0.9231 $\pm$ 0.0425} & 0.8121 $\pm$ 0.0746 & 0.8783 $\pm$ 0.0481 & 0.9195 $\pm$ 0.0473 & 0.7620 $\pm$ 0.0512 & 0.8323 $\pm$ 0.0655 & 0.5239 $\pm$ 0.0680 \\
				\midrule
				
				\multirow{2}{*}{Libras}
				& Acc. & \textbf{0.8167 $\pm$ 0.0451} & 0.7667 $\pm$ 0.0598 & 0.7444 $\pm$ 0.0911 & 0.7111 $\pm$ 0.0862 & 0.8000 $\pm$ 0.0667 & 0.5944 $\pm$ 0.0683 & 0.4194 $\pm$ 0.0685 & 0.6056 $\pm$ 0.1124 \\
				& MF1  & \textbf{0.8107 $\pm$ 0.0469} & 0.7535 $\pm$ 0.0662 & 0.7257 $\pm$ 0.0925 & 0.6756 $\pm$ 0.0982 & 0.7890 $\pm$ 0.0716 & 0.5540 $\pm$ 0.0729 & 0.3626 $\pm$ 0.0633 & 0.5523 $\pm$ 0.1151 \\
				\midrule
				
				\multirow{2}{*}{Balance}
				& Acc. & 0.9024 $\pm$ 0.0132 & 0.8799 $\pm$ 0.0301 & 0.8304 $\pm$ 0.0346 & 0.7729 $\pm$ 0.0419 & \textbf{0.9040 $\pm$ 0.0141} & 0.8080 $\pm$ 0.0287 & 0.7807 $\pm$ 0.0366 & 0.8976 $\pm$ 0.0166 \\
				& MF1  & 0.6264 $\pm$ 0.0092 & 0.6626 $\pm$ 0.0585 & 0.5939 $\pm$ 0.0180 & 0.5750 $\pm$ 0.0256 & 0.6277 $\pm$ 0.0107 & \textbf{0.7328 $\pm$ 0.0318} & 0.5415 $\pm$ 0.0252 & 0.6233 $\pm$ 0.0123 \\
				\midrule
				
				\multirow{2}{*}{Australian}
				& Acc. & 0.8493 $\pm$ 0.0566 & \textbf{0.8725 $\pm$ 0.0273} & 0.8435 $\pm$ 0.0475 & 0.8261 $\pm$ 0.0410 & 0.8464 $\pm$ 0.0455 & 0.8435 $\pm$ 0.0424 & 0.8522 $\pm$ 0.0557 & 0.8565 $\pm$ 0.0451 \\
				& MF1  & 0.8470 $\pm$ 0.0576 & \textbf{0.8709 $\pm$ 0.0280} & 0.8413 $\pm$ 0.0492 & 0.8236 $\pm$ 0.0415 & 0.8454 $\pm$ 0.0458 & 0.8424 $\pm$ 0.0419 & 0.8509 $\pm$ 0.0566 & 0.8528 $\pm$ 0.0472 \\
				\midrule
				
				\multirow{2}{*}{Banknote}
				& Acc. & 0.9993 $\pm$ 0.0022 & 0.9978 $\pm$ 0.0033 & 0.9985 $\pm$ 0.0029 & 0.9847 $\pm$ 0.0105 & \textbf{1.0000 $\pm$ 0.0000} & 0.9825 $\pm$ 0.0119 & 0.9876 $\pm$ 0.0098 & 0.9993 $\pm$ 0.0022 \\
				& MF1  & 0.9993 $\pm$ 0.0022 & 0.9978 $\pm$ 0.0034 & 0.9985 $\pm$ 0.0030 & 0.9845 $\pm$ 0.0107 & \textbf{1.0000 $\pm$ 0.0000} & 0.9823 $\pm$ 0.0120 & 0.9875 $\pm$ 0.0099 & 0.9993 $\pm$ 0.0022 \\
				\midrule
				
				\multirow{2}{*}{Rice}
				& Acc. & 0.9265 $\pm$ 0.0119 & 0.9152 $\pm$ 0.0138 & 0.9173 $\pm$ 0.0100 & 0.8856 $\pm$ 0.0133 & 0.9273 $\pm$ 0.0097 & 0.9231 $\pm$ 0.0159 & 0.9121 $\pm$ 0.0133 & \textbf{0.9278 $\pm$ 0.0116} \\
				& MF1  & 0.9248 $\pm$ 0.0121 & 0.9133 $\pm$ 0.0142 & 0.9155 $\pm$ 0.0103 & 0.8830 $\pm$ 0.0135 & 0.9257 $\pm$ 0.0098 & 0.9212 $\pm$ 0.0166 & 0.9102 $\pm$ 0.0137 & \textbf{0.9263 $\pm$ 0.0116} \\
				\midrule
				
				\multirow{2}{*}{OptDigits}
				& Acc. & 0.9811 $\pm$ 0.0046 & 0.9802 $\pm$ 0.0045 & 0.9872 $\pm$ 0.0026 & 0.9016 $\pm$ 0.0170 & \textbf{0.9888 $\pm$ 0.0031} & 0.9155 $\pm$ 0.0137 & 0.9164 $\pm$ 0.0104 & 0.9600 $\pm$ 0.0076 \\
				& MF1  & 0.9812 $\pm$ 0.0046 & 0.9803 $\pm$ 0.0045 & 0.9872 $\pm$ 0.0026 & 0.9016 $\pm$ 0.0170 & \textbf{0.9888 $\pm$ 0.0031} & 0.9154 $\pm$ 0.0135 & 0.9128 $\pm$ 0.0117 & 0.9595 $\pm$ 0.0077 \\
				\midrule
				
				\multirow{2}{*}{ISOLET}
				& Acc. & 0.9097 $\pm$ 0.0065 & 0.9483 $\pm$ 0.0092 & 0.8911 $\pm$ 0.0084 & 0.8145 $\pm$ 0.0102 & \textbf{0.9658 $\pm$ 0.0076} & 0.9112 $\pm$ 0.0090 & 0.7348 $\pm$ 0.0147 & 0.4721 $\pm$ 0.0238 \\
				& MF1  & 0.9095 $\pm$ 0.0066 & 0.9482 $\pm$ 0.0092 & 0.8904 $\pm$ 0.0087 & 0.8137 $\pm$ 0.0103 & \textbf{0.9658 $\pm$ 0.0076} & 0.9107 $\pm$ 0.0090 & 0.7277 $\pm$ 0.0164 & 0.3895 $\pm$ 0.0342 \\
				\midrule
				
				\multirow{2}{*}{HTRU2}
				& Acc. & 0.9776 $\pm$ 0.0030 & \textbf{0.9793 $\pm$ 0.0030} & 0.9782 $\pm$ 0.0029 & 0.9668 $\pm$ 0.0039 & 0.9780 $\pm$ 0.0032 & 0.9479 $\pm$ 0.0250 & 0.9744 $\pm$ 0.0036 & 0.9788 $\pm$ 0.0029 \\
				& MF1  & 0.9286 $\pm$ 0.0099 & \textbf{0.9358 $\pm$ 0.0089} & 0.9312 $\pm$ 0.0090 & 0.8999 $\pm$ 0.0125 & 0.9295 $\pm$ 0.0107 & 0.8641 $\pm$ 0.0496 & 0.9184 $\pm$ 0.0112 & 0.9325 $\pm$ 0.0093 \\
				\midrule
				
				\multirow{2}{*}{Shuttle}
				& Acc. & 0.9977 $\pm$ 0.0005 & \textbf{0.9983 $\pm$ 0.0005} & 0.9980 $\pm$ 0.0006 & 0.9981 $\pm$ 0.0004 & 0.9945 $\pm$ 0.0008 & 0.6199 $\pm$ 0.0172 & 0.9922 $\pm$ 0.0022 & 0.9982 $\pm$ 0.0003 \\
				& MF1  & 0.7135 $\pm$ 0.0550 & \textbf{0.7405 $\pm$ 0.0681} & 0.6637 $\pm$ 0.0205 & 0.7140 $\pm$ 0.0639 & 0.4740 $\pm$ 0.0256 & 0.2558 $\pm$ 0.0077 & 0.5756 $\pm$ 0.0695 & 0.6738 $\pm$ 0.0094 \\
				\midrule
				
				\multirow{2}{*}{Skin}
				& Acc. & 0.9995 $\pm$ 0.0001 & \textbf{0.9996 $\pm$ 0.0001} & 0.9995 $\pm$ 0.0001 & 0.9993 $\pm$ 0.0001 & 0.9983 $\pm$ 0.0002 & 0.7811 $\pm$ 0.0100 & 0.9978 $\pm$ 0.0006 & ME \\
				& MF1  & 0.9993 $\pm$ 0.0002 & \textbf{0.9994 $\pm$ 0.0002} & 0.9993 $\pm$ 0.0002 & 0.9990 $\pm$ 0.0002 & 0.9974 $\pm$ 0.0003 & 0.4575 $\pm$ 0.0421 & 0.9967 $\pm$ 0.0009 & ME \\
				\midrule
				
				\multirow{2}{*}{Average}
				& Acc. & \textbf{0.9309} & 0.9255 & 0.9111 & 0.8885 & 0.9291 & 0.8612 & 0.8576 & 0.8531 \\
				& MF1  & \textbf{0.8883} & 0.8859 & 0.8664 & 0.8466 & 0.8678 & 0.7993 & 0.7966 & 0.7773 \\
				\bottomrule
		\end{tabular}}
		\begin{minipage}{\textwidth}
			\footnotesize
			\vspace{0.2cm}
			\textbf{Note:} Acc. and MF1 denote Accuracy and Macro-F1, respectively. ME denotes memory error. Boldface indicates the best result for each dataset and metric. SGBSkNN is excluded from the average-value calculation on Skin Segmentation, while it is assigned the lowest rank on this dataset in rank-based analysis.
		\end{minipage}
	\end{threeparttable}
\end{table*}

\begin{table}[htbp]
	\centering
	\caption{Average ranks over the selected 18 datasets. A smaller rank indicates better performance.}
	\label{tab:avg_rank_classification}
	\begin{threeparttable}
		\scriptsize
		\resizebox{0.7\columnwidth}{!}{
			\begin{tabular}{lccc}
				\toprule
				Method & Acc. & MF1 & Avg. Rank \\
				\midrule
				MDL-GBC & \textbf{2.6944} & \textbf{2.6667} & \textbf{2.6806} \\
				XGBoost & 3.0000 & 3.1111 & 3.0556 \\
				SVM & 3.0278 & 3.2222 & 3.1250 \\
				KNN & 4.3056 & 4.4722 & 4.3889 \\
				SGBSkNN\tnote{a} & 4.5278 & 4.7500 & 4.6389 \\
				GBTSVM & 5.8333 & 5.3889 & 5.6111 \\
				CART & 5.9444 & 5.7222 & 5.8333 \\
				ScOrGBC & 6.6667 & 6.6667 & 6.6667 \\
				\bottomrule
		\end{tabular}}
	\begin{minipage}{\columnwidth}
		\footnotesize
		\vspace{0.3em}
		\tnote{a} SGBSkNN encounters a memory error on Skin Segmentation. Therefore, it is assigned the lowest rank on this dataset when computing the average ranks.
	\end{minipage}
	\end{threeparttable}
\end{table}

Table~\ref{tab:selected_18_classification_comparison} reports the dataset-level classification results, and Table~\ref{tab:avg_rank_classification} summarizes the corresponding average ranks. Overall, MDL-GBC achieves the best average Accuracy and Macro-F1 among all compared methods, with values of 0.9309 and 0.8883, respectively. It also obtains the best overall average rank, indicating that its advantage is not limited to a few individual datasets. Since SGBSkNN encounters a memory error on Skin Segmentation, its average metric values are computed over the executable datasets only, while it is assigned the lowest rank on this dataset in the rank-based comparison.

At the dataset level, MDL-GBC shows favorable performance on several small and medium-scale datasets, such as Lenses, Iris, Seeds, and Libras. In particular, on Libras, which is a high-dimensional multi-class dataset, MDL-GBC achieves the best results on both Accuracy and Macro-F1. This suggests that the proposed boundary-aware granular-ball representation can preserve useful class-discriminative structures. On large-scale datasets such as HTRU2, Shuttle, and Skin Segmentation, MDL-GBC remains competitive, although strong classical baselines such as XGBoost and SVM achieve the best results on some individual metrics.

Compared with classical baselines, MDL-GBC achieves higher average Accuracy and Macro-F1 than KNN and CART, and it obtains a slightly higher average Accuracy and a clearer Macro-F1 advantage over SVM. XGBoost remains a strong competitor and performs well on datasets such as Thyroid, Ionosphere, HTRU2, Shuttle, and Skin Segmentation. Nevertheless, MDL-GBC achieves the best overall averages and ranks, showing a more stable balance between overall accuracy and macro-level class-wise performance.

Compared with granular-ball-related baselines, MDL-GBC also shows more consistent performance. GBTSVM, ScOrGBC, and SGBSkNN can perform well on some datasets, but their average results and ranks are lower than those of MDL-GBC. This indicates that the improvement of MDL-GBC does not come merely from using granular balls, but from the proposed boundary-aware generation mechanism and the MDL-guided coding decision rule.

It is worth noting that no method dominates all datasets. Different classifiers have different inductive biases and may be more suitable for different data distributions. Even so, the overall results demonstrate that MDL-GBC provides a competitive and stable classification framework across datasets with different sample sizes, feature dimensions, class numbers, and imbalance characteristics.

\subsection{Runtime Analysis}
\label{subsec:runtime_analysis}

Since MDL-GBC explicitly constructs class-aware granular-ball representations and performs local MDL model competition, its runtime behavior is analyzed separately. The selected 18 datasets are divided into three sample-scale groups: $n\leq1000$, $1000<n\leq10000$, and $n>10000$. Table~\ref{tab:runtime_grouped_classification} reports the grouped average runtime of all compared methods, and Fig.~\ref{fig:mdl_gbc_dimension_runtime} illustrates the influence of feature dimensionality on the runtime of MDL-GBC.

\begin{table*}[htbp]
	\centering
	\caption{Grouped average runtime (seconds) across different sample-scale regimes on the selected 18 datasets.}
	\label{tab:runtime_grouped_classification}
	\begin{threeparttable}
		\scriptsize
		\resizebox{\textwidth}{!}{
			\begin{tabular}{lcccccccc}
				\toprule
				Sample scale 
				& MDL-GBC 
				& XGBoost 
				& KNN 
				& CART 
				& SVM 
				& GBTSVM 
				& ScOrGBC 
				& SGBSkNN \\
				\midrule
				$\leq 1000$ 
				& 1.6902 
				& 0.0683 
				& 0.4298 
				& \textbf{0.0020} 
				& 0.0042 
				& 0.4963 
				& 0.1863 
				& 0.7187 \\
				
				$1000 < n \leq 10000$ 
				& 31.6191 
				& 13.6600 
				& \textbf{0.0259} 
				& 1.0623 
				& 0.5613 
				& 1.3286 
				& 40.0234 
				& 127.6938 \\
				
				$>10000$ 
				& 2546.0890 
				& 0.4335 
				& 0.2548 
				& \textbf{0.1219} 
				& 16.5419 
				& 1.0222 
				& 318.2045 
				& 7156.0405\tnote{a} \\
				
				\bottomrule
		\end{tabular}}
		\begin{minipage}{\textwidth}
			\footnotesize
			\vspace{0.3em}
			\tnote{a} SGBSkNN encounters a memory error on Skin Segmentation. Therefore, its runtime average in the large-scale group is computed over the executable large-scale datasets only.
		\end{minipage}
	\end{threeparttable}
\end{table*}

\begin{figure*}[!t]
	\centering
	\includegraphics[width=\textwidth]{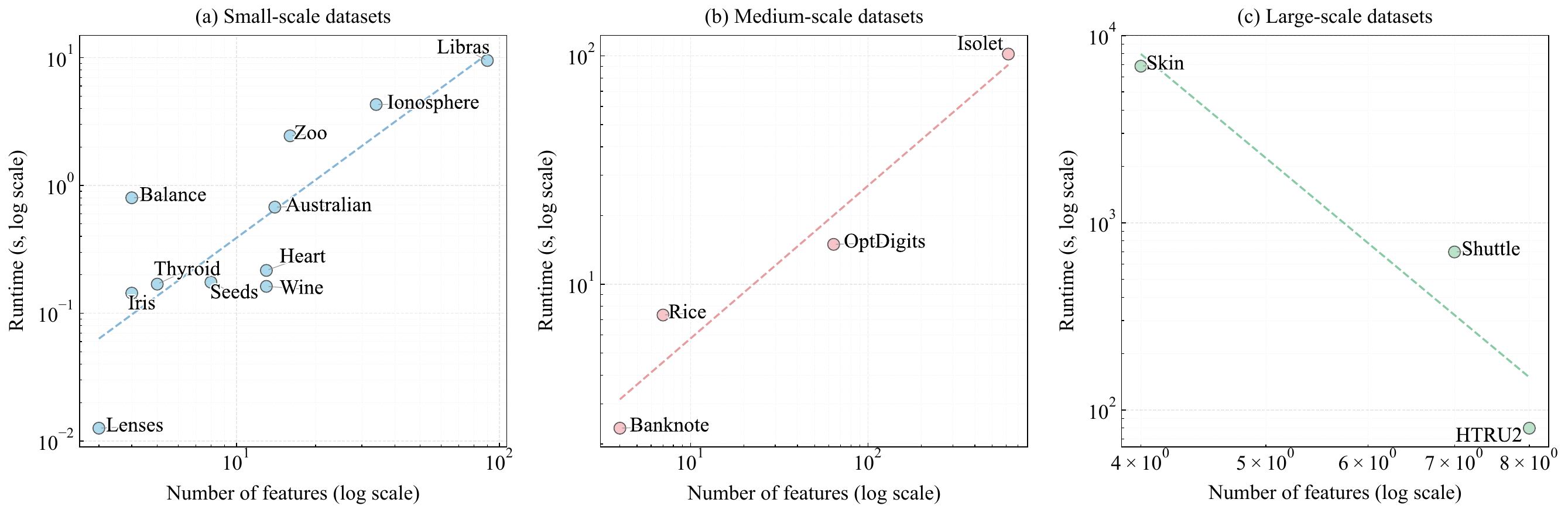}
	\caption{Effect of feature dimensionality on the runtime of MDL-GBC under different sample-scale regimes.}
	\label{fig:mdl_gbc_dimension_runtime}
\end{figure*}

As shown in Table~\ref{tab:runtime_grouped_classification}, the runtime of MDL-GBC increases with the sample scale. Its average runtime rises from 1.6902 seconds on small-scale datasets to 31.6191 seconds on medium-scale datasets, and further to 2546.0890 seconds on large-scale datasets. This trend is expected because recursive granular-ball generation, nearest-negative evidence evaluation, and class-level coding-cost computation become more expensive as the number of training samples increases.

Compared with classical baselines, MDL-GBC is not the fastest method. CART and KNN are much faster in most sample-scale groups, and XGBoost also remains computationally efficient due to its optimized implementation. By contrast, MDL-GBC introduces additional cost to obtain a non-parametric, boundary-aware, and interpretable granular-ball representation. Among granular-ball-related baselines, GBTSVM is relatively efficient, whereas ScOrGBC and especially SGBSkNN become more time-consuming as data scale increases. SGBSkNN also encounters a memory error on Skin Segmentation, indicating limited scalability in this setting.

Figure~\ref{fig:mdl_gbc_dimension_runtime} further shows that runtime is affected not only by sample size, but also by feature dimensionality. For example, Libras and ISOLET require more computation than lower-dimensional datasets with comparable sample sizes, mainly because distance computation, local statistic estimation, and coding-cost evaluation all depend on the feature dimension. In the large-scale group, however, sample size becomes the dominant factor: Skin Segmentation has only four features, but its extremely large sample size leads to the highest runtime of MDL-GBC.

Overall, MDL-GBC trades part of computational efficiency for adaptive granularity selection, boundary-aware representation, and interpretability. Future work will focus on improving efficiency through approximate nearest-neighbor search, compact ball indexing, parallel MDL evaluation, and faster prediction strategies.

\section{Discussion and Conclusions}
\label{sec:discussion_conclusion}

This paper proposed MDL-GBC, a boundary-aware non-parametric and interpretable granular-ball classifier based on the Minimum Description Length principle. The central idea is to regard class-conditional granular-ball construction as a local model selection problem guided by boundary evidence. For each class-conditional granular ball, MDL-GBC compares three candidate explanations under a unified coding criterion: the single-ball model, the two-ball model, and the core-boundary model. In this way, ball retention, geometric splitting, and boundary-sensitive refinement are determined by description-length comparison rather than by manually specified purity thresholds or fixed granularity parameters.

This formulation provides a unified way to connect class-conditional representation learning with classification-oriented boundary modeling. For each target class, positive samples describe the class evidence, while samples from the remaining classes provide negative boundary evidence. As a result, the learned granular balls can capture not only within-class compactness, but also cross-class boundary exposure. The prediction stage follows the same coding perspective: stable balls of the same class are aggregated through a class-level mixture coding rule, so the final decision is made by comparing class-wise coding costs instead of relying on a single nearest ball or a purely local voting rule.

The experimental results on 18 benchmark datasets show that MDL-GBC achieves competitive and stable classification performance. It obtains the best average Accuracy, Macro-F1, and overall average rank among the compared methods. At the dataset level, MDL-GBC performs favorably on small and medium-scale datasets such as Lenses, Iris, Seeds, and Libras, and remains competitive on large-scale datasets such as HTRU2, Shuttle, and Skin Segmentation. These results suggest that the proposed MDL-based construction and prediction mechanism can provide a useful balance between overall accuracy and macro-level class-wise performance.

Compared with classical baselines, MDL-GBC offers an interpretable region-level representation. While KNN, CART, SVM, and XGBoost rely on sample-level neighborhoods, tree partitions, margin-based boundaries, or boosted ensembles, MDL-GBC represents each class by a set of stable granular balls selected through local MDL competition. Compared with existing granular-ball-related methods, its main advantage lies in the consistency between representation construction and prediction. Local structural refinement is governed by description-length comparison, and final classification is also formulated through coding-cost comparison. This consistency makes the learned representation more transparent and better aligned with the classification objective.

The runtime analysis shows that MDL-GBC introduces additional computational cost. This cost mainly comes from recursive granular-ball construction, nearest-negative evidence evaluation, and local MDL model competition. Sample size is the primary factor driving runtime growth, while feature dimensionality further increases the cost of distance computation, local statistic estimation, and coding-cost evaluation. Therefore, MDL-GBC should not be viewed as the fastest classifier, but rather as a method that trades part of computational efficiency for non-parametric granularity selection, boundary-aware representation, and interpretability.

Several limitations also remain. First, distance-based local modeling and nearest-negative evidence evaluation can become expensive on very large datasets. Second, the diagonal Gaussian coding model provides a simple and efficient approximation, but may not fully capture highly correlated feature spaces. Third, the current method mainly relies on Euclidean geometry after Min--Max normalization, which may be insufficient for mixed-type, structured, or domain-specific data. These limitations indicate that further improvements are needed in both computational efficiency and modeling flexibility.

Future work will therefore focus on accelerating MDL-GBC and extending its modeling capacity. Possible directions include approximate nearest-neighbor search, ball indexing, parallel local MDL evaluation, and more expressive coding models such as full-covariance, low-rank, kernelized, or distribution-adaptive coding terms. Incremental and online extensions are also worth exploring, so that stable granular balls can be updated when new samples arrive rather than regenerated from scratch.

In conclusion, MDL-GBC provides a boundary-aware, non-parametric, and interpretable granular-ball classifier based on the Minimum Description Length principle. By unifying local construction decisions through MDL-based model competition and aligning prediction with class-level mixture coding, the proposed method reduces dependence on heuristic structural thresholds and gives the learned granular-ball representation a clear coding-theoretic explanation. The experimental results support the effectiveness of MDL-guided boundary-aware granular-ball learning for classification.

\section*{CRediT authorship contribution statement}
\textbf{Zeqiang Xian}: Conceptualization, Methodology, Writing – original draft, Software, Validation. \textbf{Caihui Liu}: Conceptualization, Methodology, Writing – review \& editing, Validation, Supervision. \textbf{Yong Zhang}: Software, Data curation. 
\textbf{Wenjing Qiu}: Software, Data curation. 
\textbf{Duoqian Miao}: Writing – review \& editing
\textbf{Witold Pedrycz}: Writing – review \& editing

\section*{Declaration of Competing Interest}
The authors declare that they have no known competing financial interests or personal relationships that could have appeared to influence the work reported in this paper.

\section*{Data availability}
Data will be made available on request.

\section*{Acknowledgment}
The research is supported by the National Natural Science Foundation of China under Grant No. 62566003, Graduate Innovation Funding Program of Jiangxi Province under Grant No. YC2025-S224.

\bibliographystyle{IEEEtran}
\bibliography{References}

\end{document}